\documentclass{article}

\usepackage{microtype}
\usepackage{graphicx}
\usepackage{subfigure}
\usepackage{booktabs} 

\usepackage{hyperref}
\usepackage{multirow}
\usepackage{threeparttable}

\usepackage[accepted]{icml2024}

\usepackage{amsmath}
\usepackage{amssymb}
\usepackage{mathtools}
\usepackage{amsthm}
\usepackage[shortlabels]{enumitem}
\usepackage{adjustbox}
\usepackage{subcaption}

\usepackage[capitalize,noabbrev]{cleveref}

\theoremstyle{plain}
\newtheorem{theorem}{Theorem}[section]
\newtheorem{proposition}[theorem]{Proposition}

\newtheorem{example}[theorem]{Example}

\theoremstyle{definition}
\newtheorem{definition}[theorem]{Definition}

\theoremstyle{remark}

\usepackage[textsize=tiny]{todonotes}

\icmltitlerunning{Lyapunov-stable Neural Control for State and Output Feedback}

\begin{document}

\twocolumn[
\icmltitle{Lyapunov-stable Neural Control for State and Output Feedback:\\ A Novel Formulation}



\icmlsetsymbol{equal}{*}

\begin{icmlauthorlist}
\icmlauthor{Lujie Yang$^*$}{mit}
\icmlauthor{Hongkai Dai$^*$}{tri}
\icmlauthor{Zhouxing Shi}{ucla}
\icmlauthor{Cho-Jui Hsieh}{ucla}
\icmlauthor{Russ Tedrake}{mit,tri} 
\icmlauthor{Huan Zhang}{uiuc}
\end{icmlauthorlist}

\icmlaffiliation{mit}{MIT}
\icmlaffiliation{tri}{Toyota Research Institute}
\icmlaffiliation{ucla}{UCLA}
\icmlaffiliation{uiuc}{UIUC}

\icmlcorrespondingauthor{Lujie Yang}{lujie@mit.edu}
\icmlcorrespondingauthor{Huan Zhang}{huan@huan-zhang.com}

\icmlkeywords{Machine Learning, ICML}

\vskip 0.3in

]



\printAffiliationsAndNotice{\icmlEqualContribution} 

\begin{abstract}
Learning-based neural network (NN) control policies have shown impressive empirical performance in a wide range of tasks in robotics and control. However, formal (Lyapunov) stability guarantees over the region-of-attraction (ROA) for NN controllers with nonlinear dynamical systems are challenging to obtain, and most existing approaches rely on expensive solvers such as sums-of-squares (SOS), mixed-integer programming (MIP), or satisfiability modulo theories (SMT). In this paper, we demonstrate a new framework for learning NN controllers together with Lyapunov certificates using fast empirical falsification and strategic regularizations. 
We propose a novel formulation that defines a larger verifiable region-of-attraction (ROA) than shown in the literature, and refines the conventional restrictive constraints on Lyapunov derivatives to focus only on certifiable ROAs. The Lyapunov condition is rigorously verified post-hoc using branch-and-bound with scalable linear bound propagation-based NN verification techniques. The approach is efficient and flexible, and the full training and verification procedure is accelerated on GPUs without relying on expensive solvers for SOS, MIP, nor SMT.
The flexibility and efficiency of our framework allow us to demonstrate Lyapunov-stable output feedback control with synthesized NN-based controllers and NN-based observers with formal stability guarantees, for the first time in literature. Source code at \href{https://github.com/Verified-Intelligence/Lyapunov_Stable_NN_Controllers}{github.com/Verified-Intelligence/Lyapunov\_Stable\_NN\_Controllers}
\end{abstract}
\vspace*{-0.4cm}
\begin{figure}
\centering
	\includegraphics[width=0.48\textwidth]{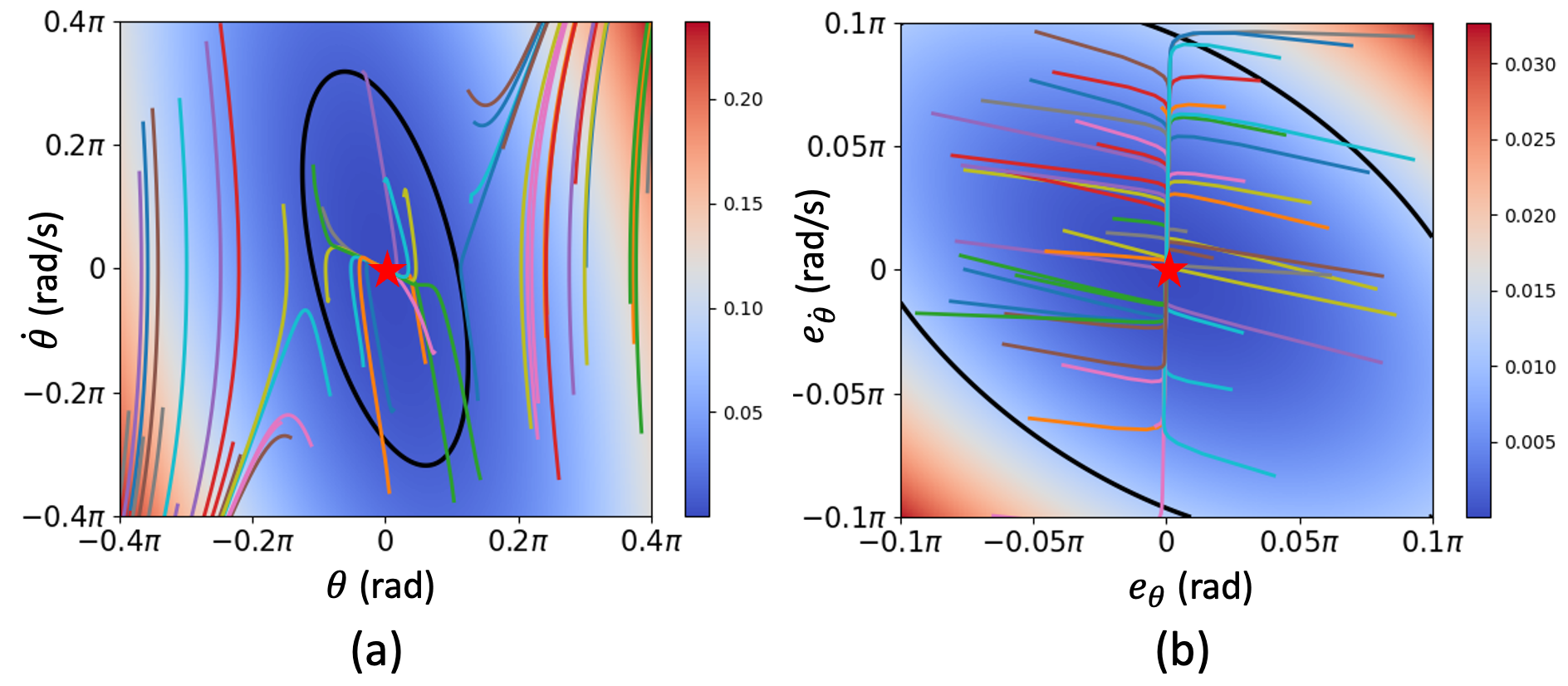}
	\caption{Phase portrait (colorful trajectories) of simulating the angle-observed pendulum with the synthesized neural-network controller and observer in different 2-dimensional slice. The torque limit $|u|\le \frac{mgl}{3}$ is challenging for both synthesis and verification. The certified ROA is outlined by the black contour. To the best of our knowledge, our method provides the first formal neural certificate for the  pendulum with output feedback control.
 }
	\label{fig:pendulum_output}
	\vspace*{-0.4cm}
\end{figure}
\section{Introduction}
Deep learning has significantly advanced the development of neural-network-based controllers for robotic systems, especially those leveraging output feedback from images and sensor data \cite{kalashnikov2018qt, zhang2016learning}. Despite their impressive performance, many of these controllers lack the critical stability guarantees that are essential for safety-critical applications. Addressing this issue, Lyapunov stability \cite{lyapunov1992general} in control theory offers a robust framework to ensure the closed-loop stability of nonlinear dynamical systems. Central to this theory, a Lyapunov function is a scalar function whose value decreases along the system's closed-loop trajectories, guiding the system towards a stable equilibrium from any states within the region-of-attraction (ROA). While existing research has successfully synthesized  \emph{simple} (linear or polynomial) controllers \cite{tedrake2010lqr,  majumdar2013control, yang2023approximate, dai2023convex} with rigorous stability guarantees, certified by Lyapunov functions and their sublevel sets as region-of-attraction \cite{slotine1991applied} using linear quadratic regulator (LQR) or sum-of-squares (SOS) \cite{parrilo2000structured} based method, a gap persists in extending these guarantees to more complex neural-network controllers. To bridge this gap, we aim to synthesize neural-network controllers with Lyapunov functions, in order to certify the stability of the closed-loop system with either state or output feedback. 

Many recent works have considered searching for Lyapunov or barrier functions using sampled data to guide the synthesis of neural-network controllers \cite{dawson2022safe, jin2020neural, liu2023safe, sun2021learning}. Although empirically successful in a diverse range of tasks, they do not yet provide formal guarantees. In contrast, other research \cite{dai2021lyapunov, wu2023neural, everett2021reachability, vincent2022reachable} focuses on the rigorous certification, which is grounded in formal methods \cite{liu2021algorithms, edwards2023general}, with tools such as  Satisfiable Modulo Theories (SMT) \cite{chang2019neural, abate2020formal}, Mixed-integer Programming (MIP) \cite{dai2021lyapunov, chen2021learning} or Semidefinite Programming (SDP) \cite{wang2023joint, yin2021stability, fazlyab2020safety}. These formal methods formulate the Lyapunov certification problem as proving that certain functions (the Lyapunov function itself, together with the negation of its time derivative) are always non-negative over a domain. The state-of-the-art SMT solvers \cite{gao2013dreal,de2008z3} become limited by the complexity of the functions they can certify, especially when the controller, dynamics, sensor output function, observer, and the Lyapunov function intertwine. Consequently, the SMT-based approaches only synthesized simple controllers \cite{chang2019neural}. On the other hand, MIP solvers \cite{bertsimas1997introduction} employ a branch-and-bound process and divide the verification problem into linear subproblems. This approach has better scalability to higher dimensional systems with neural-network controllers \cite{dai2021lyapunov}, with the limitation of requiring the original system dynamics to be approximated as piecewise linear functions; hence, it cannot handle generic nonlinear dynamical systems. Due to these limitations in scalability, previous neural-network Lyapunov control works predominantly provided guarantees only for state-feedback control. Our work addresses the more challenging but practically relevant domain of output-feedback control, identifying and overcoming the limitations of existing methods to synthesize and certify controllers for real-world applications.

In addition to relying on resource-intensive solvers for SMT, MIP or SDP, prior works on neural certificates \cite{chang2019neural, dai2021lyapunov, wu2023neural} imposed the Lyapunov derivative constraint across an entire explicitly defined region, rather than the implicitly defined region-of-attraction. This results in unnecessarily restrictive conditions over uncertified regions. Moreover, all of them failed to find the largest certifiable ROA by applying incorrect restrictions on the ROA. We remedy these issues with a new formulation in Sec.~\ref{sec:verification_formulation} that eliminates the overly stringent constraints on the Lyapunov time-derivative over uncertifiable regions. 

To achieve the ambitious goal of synthesizing Lyapunov-stable neural control for general nonlinear dynamical systems with both state and output feedback, our work utilizes the latest progress from the neural network verification community. Recently, $\alpha,\!\beta$-CROWN \cite{zhang2018efficient, xu2020automatic,xu2020fast, wang2021beta, zhang2022general,shi2023formal} demonstrated great scalability in robustness verification of large-scale computer vision neural networks and safety verification of neural-network controllers~\citep{everett2023drip,mathiesen2022safety,rober2023backward,kotha2024provably}. This complete verifier has a few distinct features that are specialized for verifying NN-controlled systems. First, it \emph{exploits the network structure} of the underlying verification problem by efficiently propagating linear bounds through neural networks; in contrast, general-purpose MIP or SMT solvers do not effectively exploit the rich NN structure. Second, the bound propagation process is \emph{GPU-friendly}, allowing the efficient verification of large networks and the fast evaluation of many subproblems using branch-and-bound.
\looseness=-1


Our key contributions include:
 \begin{itemize}[topsep=0pt,parsep=0pt,partopsep=0pt,wide,labelwidth=!,labelindent=0pt]
     \item We synthesize and verify neural-network \emph{controllers}, \emph{observers} together with Lyapunov functions for \emph{general nonlinear} dynamical systems. To the best of our knowledge, this is the first work to achieve this goal with formal guarantees.
     \item We propose a novel formulation that defines a large certifiable region-of-attraction (see Fig. \ref{fig:pendulum_output}) and removes the unnecessarily restrictive Lyapunov time-derivative constraints in uncertified regions. Compared with previous works, our new formulation is easier to train and certify, while affording control over the growth of the ROA during training. 
     \item Unlike previous work with formal guarantees~\cite{dai2021lyapunov,chang2019neural}, which guided training with expensive verifiers like SMT or MIP,
     we show that cheap adversarial attacks with strategic regularization are sufficient to guide the learning process and achieve a \emph{certified} ROA via \emph{post-training} verification using a strong verifier.
 \end{itemize}
 
The paper is organized as follows. In Sec.\ref{section: formulation}, we discuss the problem formulation and our parameterization of the controllers/observers using NNs. In Sec.\ref{section: method}, we present our new formulation to verify Lyapunov stability and our new training algorithm to synthesize controllers. In Sec.\ref{section:results}, we demonstrate that our novel formulation leads to larger ROAs compared to the state-of-the-art approaches in multiple dynamical systems. For the first time in literature, we present verified neural network controllers and observers for pendulum and 2D quadrotor with \emph{output feedback} control.%

\vspace{-0.5em}
\section{Problem Statement}
\label{section: formulation}
\vspace{-0.3em}
We consider a nonlinear discrete-time plant
\vspace{-0.5em}
\begin{subequations}
\begin{align}
x_{t+1} = f(x_t, u_t)\\
y_t = h(x_t)
\end{align}
\end{subequations}
where $x_t\in\mathbb{R}^{n_x}$ is the state, $u_t\in \{u|u_{\text{lo}} \le u \le u_{\text{up}}\}\subset \mathbb{R}^{n_u}$ is the control input, and $y_t \in\mathbb{R}^{n_y}$ is the system output. We denote the goal state/control at equilibrium as $x^*/u^*$ and assume that $f$ is continuous.

Our objective is to jointly search for a Lyapunov function and a neural-network control policy (together with a neural-network state observer for output feedback scenarios) to formally certify the Lyapunov stability of the closed-loop system. Moreover, we aim to train the policy that maximizes the region-of-attraction (ROA) for the closed-loop system and certify its inner-approximation. We will first introduce our parameterization of the policy and the state observer, and then specify our training and verification goal.

\textbf{State feedback control.\enskip}
In this scenario, the controller has full access to the accurate state $x_t$. We parameterize the control policy with a neural network $\phi_\pi:\mathbb{R}^{n_x}\rightarrow\mathbb{R}^{n_u}$ as
\begin{equation}
	u_t = \pi(x_t) = \text{clamp}\left(\phi_\pi(x_t) - \phi_\pi(x^*) + u^*, u_{\text{lo}}, u_{\text{up}}\right). \label{eq:control_network}
\end{equation}
By construction, this control policy $\pi(\bullet)$ produces the goal control $u^*$ at the goal state $x^*$.

\textbf{Output feedback control.\enskip}
In the output feedback setting, the controller does not have access to the true state $x_t$ but rather only observes the output $y_t$. The output can be either a subset of the state or, more generally, a nonlinear function of $x_t$. In this paper, we consider the situation where there are only uncertainties in the initial conditions. We aim to estimate the state as $\hat x_t$ with a dynamic state observer using a neural network $\phi_{\text{obs}}:\mathbb{R}^{n_x}\times\mathbb{R}^{n_y}\rightarrow\mathbb{R}^{n_x}$ as \looseness=-1
\begin{equation}
\hat x_{t+1} = f(\hat x_t, u_t) + \phi_{\text{obs}}(\hat x_t, y_t-h(\hat x_t))-\phi_{\text{obs}}(\hat x_t, \mathbf{0}_{n_y}),
\end{equation}
where $\mathbf{0}_{n_y}\in\mathbb{R}^{n_y}$ is a vector of all 0s. Notice that this state observer resembles the Luenberger observer \cite{luenberger1971introduction}, where the observer gain is replaced by the neural network $\phi_{\text{obs}}$. By construction, if $\hat{x}_{t} = x_{t}$, then our observer ensures that $\hat{x}_{t+1} = x_{t+1}$. The network $\phi_\pi:\mathbb{R}^{n_x} \times \mathbb{R}^{n_y}\rightarrow\mathbb{R}^{n_u}$ now takes in both the state estimate $\hat x_t$ and output $y_t$ rather than the true state $x_t$
\begin{align}
\vspace*{-0.2cm}
    u_t &= \pi(\hat x_{t}, y_t) \nonumber \\ &= \text{clamp}\left(\phi_\pi(\hat x_t, y_t) - \phi_\pi(x^*, h(x^*)) + u^*, u_{\text{lo}}, u_{\text{up}}\right).
    \vspace*{-0.2cm}
\end{align}
Unlike linear dynamical systems whose optimal output feedback controller only depends on the estimated state $\hat x_t$ (i.e., the separation principle \cite{aastrom2012introduction, athans1971role}), we expand the design of our neural-network controller to depend on both $\hat x_t$ and $y$ for the nonlinear dynamical systems. 
By also incorporating the output $y_t$, we enable the controller to distinguish and appropriately react to different actual states $x_t$ that may correspond to the same state estimate. We find this controller design to be sufficient for all our experiments.

\textbf{Unifying state and output feedback notation.}
To unify the design for both state and output feedback control and simplify the notation, we introduce an internal state $\xi_t \in\mathbb{R}^{n_\xi}$ with the closed-loop dynamics
\begin{equation}
    \xi_{t+1} = f_{\text{cl}}(\xi_t). \label{eq:internal_cl_dynamics}
\end{equation}
For state feedback, the internal state is simply the true state $\xi_t = x_t$ and the closed-loop dynamics is
\begin{align}
    f_{\text{cl}}(\xi_t) = f(\xi_t, \pi(\xi_t)). 
\end{align}
For output feedback, we define the state prediction error $e_t = \hat{x}_t - x_t$,
whose value at the equilibrium is required to be $e^* \equiv \mathbf{0}_{n_x}$. The internal state is defined as $\xi_t = \left [x_t, \ e_t \right ]^\top$
with closed-loop dynamics
\begin{subequations}
\begin{align}
     &f_{\text{cl}}(\xi_t) = \begin{bmatrix} f(x_t, \pi(\hat x_t, h(x_t))) \\ f(\hat x_t, \pi(\hat x_t, h(x_t))) + g(x_t, \hat x_t) - x_t\end{bmatrix}\\
     &g(x_t, \hat x_t) = \phi_{\text{obs}}(\hat x_t, h(x_t)-h(\hat x_t))-\phi_{\text{obs}}(\hat x_t, \mathbf{0}_{n_y}).
\end{align}
\end{subequations}
\begin{definition}[region-of-attraction] The region-of-attraction for an equilibrium state $\xi^*$ is the largest invariant set $\mathcal{R}$ such that under the closed-loop dynamics \eqref{eq:internal_cl_dynamics}, $\lim_{t\rightarrow \infty}\xi_t =\xi^*$ for all $\xi_0 \in \mathcal{R}$.    
\end{definition}
\textbf{Training and verification goal.} 
Formally, we aim at finding a Lyapunov function $V(\xi_t):\mathbb{R}^{n_\xi}\rightarrow\mathbb{R}$, and an invariant and bounded set $\mathcal{S}$ that contains the goal state at the equilibrium $\xi^*$ as a certified inner-approximation of the ROA $\mathcal{R}$. Our objective is formalized in the optimization problem
\begin{subequations}
\begin{align}
    \max_{V, \pi, \phi_{\text{obs}}} \; &\text{Vol}(\mathcal{S})\label{eq:roa_volume_objective}\\
    \text{s.t } &V(\xi_t) > 0 \; \forall \xi_t\ne \xi^* \in\mathcal{S}\label{eq:lyapunov_positivity}\\
	&V(\xi_{t+1}) - V(\xi_t) \le -\kappa V(\xi_t) \; \forall \xi_t\in\mathcal{S}\label{eq:lyapunov_derivative}\\
	&V(\xi^*) = 0,\label{eq:lyapunov_equilibrium}
\end{align}\label{eq:lyapunov}%
\end{subequations}
where $\kappa > 0$ is a fixed constant for exponential stability convergence rate. Constraints \eqref{eq:lyapunov_positivity}-\eqref{eq:lyapunov_equilibrium} guarantee that trajectories originating from any state within $\mathcal{S}$ will eventually converge to the goal state $\xi^*$. Hence, $\mathcal{S}$ is an inner-approximation of the ROA. Our subsequent efforts will focus on expanding this set $\mathcal{S}$ for broader stability guarantees. \looseness=-1


\section{Methodology}
\label{section: method}
Previous works on verified neural certificates \cite{chang2019neural, dai2021lyapunov, wu2023neural} enforced overly restrictive Lyapunov derivative constraints in an entire explicitly defined region, and failed to find the largest verifiable ROA. In this section, we present our new formulation that defines a larger certifiable ROA and removes these constraints outside the ROA. 
We then discuss our verification and training algorithms to generate stabilizing controllers and observers together with Lyapunov certificates.
\subsection{Design of learnable Lyapunov functions}
\begin{figure}
    \centering
    \includegraphics[width=0.38\textwidth]{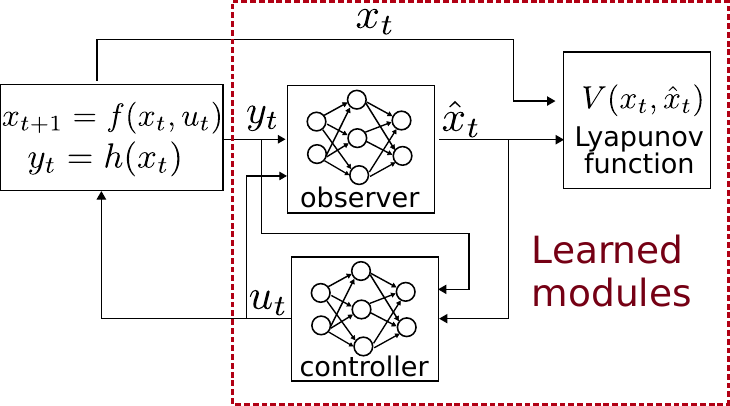}
    \caption{Given a dynamical system, we find an observer, a controller, and a Lyapunov function (parametrized by functions such as NNs), to prove the stability of the closed-loop system with a large certified region-of-attraction.\looseness=-1}
    \label{fig:feedback_diagram_observer}
    \vspace*{-0.3cm}
\end{figure}
To enforce the positive definite condition \eqref{eq:lyapunov_positivity}, we adopt two types of parameterizations for the Lyapunov function. We construct the Lyapunov function using either 1) a neural network $\phi_V:\mathbb{R}^{n_\xi}\rightarrow\mathbb{R}$ as
\begin{equation}
	V(\xi_t) = |\phi_V(\xi_t) - \phi_V(\xi^*)| + \|(\epsilon I + R^TR)(\xi_t - \xi^*)\|_1, \label{eq:nn_lyapunov}
\end{equation}
or 2) a quadratic function
\begin{align}
    V(\xi_t) = (\xi_t - \xi^*)^T(\epsilon I + R^TR)(\xi_t - \xi^*),
    \label{eq:quadratic_lyapunov}
\end{align}
where $\epsilon$ is a small positive scalar and $R \in \mathbb{R}^{n_\xi \times n_\xi}$ is a learnable matrix parameter to be optimized. Notice that since $\epsilon I + R^TR$ is a  strictly positive definite matrix, the term $|(\epsilon I + R^TR)(\xi_t - \xi^*)|_1$ or $(\xi_t - \xi^*)^T(\epsilon I + R^TR)(\xi_t - \xi^*)$ guarantees the Lyapunov candidate to be strictly positive when $\xi_t \neq \xi^*$. Also, by construction, the Lyapunov candidates \eqref{eq:nn_lyapunov} and \eqref{eq:quadratic_lyapunov} satisfy $V(\xi^*) = 0$ (condition \eqref{eq:lyapunov_equilibrium}). We illustrate our entire system diagram in Fig. \ref{fig:feedback_diagram_observer}.

\subsection{A Novel Verification Formulation}
\label{sec:verification_formulation}
Our verifier $\alpha,\!\beta$-CROWN, along with others like  dReal, can verify statements such as $V(\xi_{t+1}) - V(\xi_t)\le -\kappa V(\xi_t)$, over an \textit{explicitly} defined region. 
Therefore, we choose a compact ``region-of-interest" $\mathcal{B}$ (e.g., a box) containing the equilibrium state $\xi^*$, and constrain $\mathcal{S}$ as the intersection of $\mathcal{B}$ and a sublevel set of $V$ as
\begin{equation} \label{eq:S_def}
     \mathcal{S} \coloneqq \{\xi_t \in \mathcal{B}| V(\xi_t) < \rho\},
\end{equation}
where $\rho$ ensures
\begin{equation}\label{eq:rho_cond}
    \xi_{t+1} \in \mathcal{B} \quad \forall \xi_t \in \mathcal{S}.
\end{equation}
\begin{proposition}\label{prop:S_invariance}
If $\xi_t \in \mathcal{S}$, then $\xi_{t+1} \in \mathcal{S}$. Moreover, $\mathcal{S} \subseteq \mathcal{R}$.
\end{proposition}
\begin{proof}
    For any $\xi_t$ in $\mathcal{S}$, we have that $V(\xi_{t+1}) - V(\xi_t) \le -\kappa V(\xi_t) \le 0$ by \eqref{eq:lyapunov_derivative} and therefore $V(\xi_{t+1}) \le V(\xi_t) < \rho$. By \eqref{eq:rho_cond}, we also know that $\xi_{t+1} \in \mathcal{B}$. Therefore, we have that $\xi_{t+1} \in \mathcal{S}$ and $\mathcal{S}$ is invariant.
\end{proof}

\begin{figure}[t]
     \centering
     \includegraphics[width=0.3\textwidth]{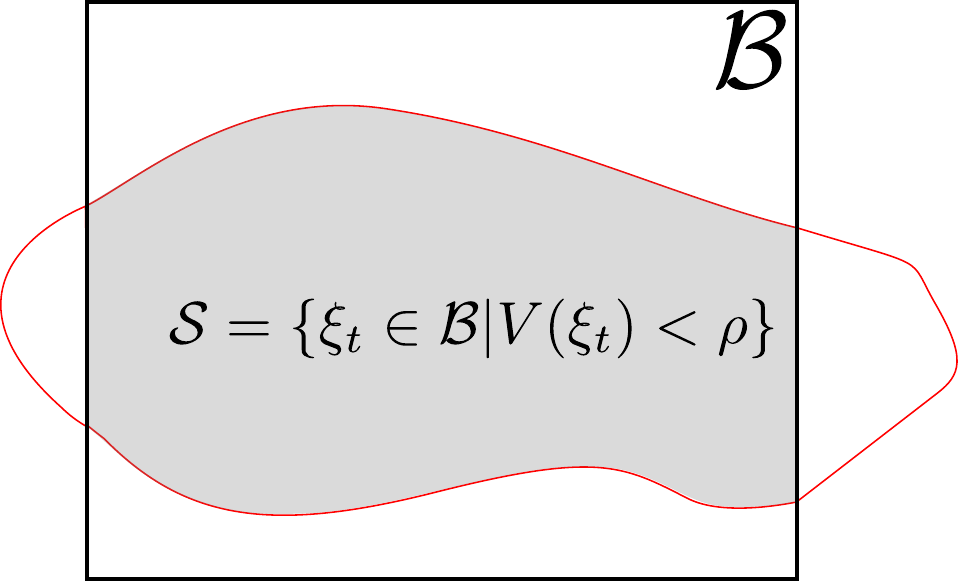}
     \caption{We choose a ``region-of-interest" $\mathcal{B}$. The invariant set $\mathcal{S}$ (the shaded region) is the intersection of the sub-level set $\{\xi_t | V(\xi_t)<\rho\}$ (the red curve) and the region $\mathcal{B}$. $\rho$ is chosen such that $\xi_{t+1}\in \mathcal{B}\; \forall \xi_t\in\mathcal{S}$.}
     \label{fig:box_roa}
\end{figure}

Fig. \ref{fig:box_roa} illustrates the region of interest $\mathcal{B}$, and the invariant set $\mathcal{S}$. Taking $\mathcal{S}$ to be the intersection of the sublevel set and $\mathcal{B}$ is particularly important here, since smaller $\mathcal{B}$ with relatively large $\mathcal{S}$ reduces the burden on our $\alpha,\!\beta$-CROWN verifier compared to larger $\mathcal{B}$ with relatively small $\mathcal{S}$.

\textbf{Drawbacks of existing approaches.} Many complete verifiers, including $\alpha,\!\beta$-CROWN and MIP solvers, do not directly handle verification over an \textit{implicitly} defined region, such as our invariant set $\mathcal{S}$ that depends on $V$. To get around this limitation and enforce the derivative condition \eqref{eq:lyapunov_derivative}, previous works \cite{chang2019neural, dai2021lyapunov, wu2023neural} typically adopt a two-step approach to obtain the region-of-attraction. In step 1, they synthesize and verify the Lyapunov derivative condition over the entire explicitly defined domain
\begin{equation}
    V(\xi_{t+1}) - V(\xi_t) \le -\kappa V(\xi_t) \; \ \forall \xi_t\in\mathcal{B}\label{eq:lyapunov_derivative_box}.
\end{equation}
In step 2, they compute the largest sublevel set \emph{within} $\mathcal{B}$, denoted by $\Tilde{\mathcal{S}} \coloneqq \{\xi_t \in \mathcal{B}|V(\xi_t) < \min_{\bar{\xi}_t\in\partial \mathcal{B}}V(\bar{\xi}_t)\}$, as the certified ROA. The drawback of this two-step approach is twofold: 1) the Lyapunov derivative condition is unnecessarily certified over $\mathcal{S}^c \cap \mathcal{B}$. This region, represented as the unshaded region \textit{outside} of $\mathcal{S}$ and within $\mathcal{B}$ in Fig. \ref{fig:box_roa}, is not invariant. This means states initiating from this region have no guarantees on stability or convergence. 2) Their certified ROA $\Tilde{\mathcal{S}}$ is not guaranteed to be invariant and is much smaller than the largest possible $\mathcal{S}$ by construction. Consequently, this two-step approach makes the synthesis and verification unnecessarily hard, significantly reducing the size of the certified ROA.
\begin{example}
    Consider a 2-dimensional single integrator plant $x_{t+1} = x_t  + 0.1 \cdot u_t$ with the controller $\pi(x_t)=-x_t$ and Lyapunov function $V(x_t) = x_t^T x_t$. Let $\mathcal{B} = [-1, 1]^2$ and $\kappa = 0.1$, the Lyapunov derivative condition is satisfied on the entire set $\mathcal{B}$. Moreover, since the closed-loop dynamics lead to $x_{t+1} = (1-0.1)\cdot x_t \in \mathcal{B}, \forall x_t \in \mathcal{B}$, we have that $\mathcal{S} = \mathcal{B}$.
    However, previous works can only find the largest sublevel set as the circle $\Tilde{\mathcal{S}} = \{x_t|x_t^Tx_t \le 1\}$, which is strictly contained in $\mathcal{B}$.
\end{example}
\textbf{A new formulation for verifying ROA.} To overcome the limitations of existing approaches, we describe how to reformulate the derivative condition \eqref{eq:lyapunov_derivative}, originally defined over $\mathcal{S}$, to be verified over the explicitly defined region $\mathcal{B}$. 
\begin{theorem}\label{theorem:verification} Let $F(\xi_t) := V(f_{\text{cl}}(\xi_t)) - (1-\kappa)V(\xi_t)$. If the condition 
\begin{equation}
\label{eq:verification}
(-F(\xi_t) \ge 0  \land  \xi_{t+1}\in \mathcal{B}) \ \lor \ (V(\xi_t) \ge \rho), \ \forall \xi_t\in\mathcal{B}
\end{equation}
holds, then the closed-loop system \eqref{eq:internal_cl_dynamics} is Lyapunov stable with $\mathcal{S}$ as the certified ROA.
\end{theorem}

Namely a point $\xi_t\in\mathcal{B}$ either satisfies the Lyapunov function value decreasing and will stay in $\mathcal{B}$ at the next time step, or is outside of the certified ROA $\mathcal{S}$.
Adding the condition $V(\xi_t) \ge \rho$ makes verification \emph{easier}
, as the verifier only needs to check the Lyapunov derivative condition when $\xi_t$ is within the sublevel set $V(\xi_t) < \rho $.



\textbf{Verification with $\alpha,\!\beta$-CROWN.} The verification problem~\eqref{eq:verification} is presented as a general computation graph to $\alpha,\!\beta$-CROWN, and we extended the verifier to support all nonlinear operations in our system, such as trigonometric functions. Initially, $\alpha,\!\beta$-CROWN uses an efficient bound propagation method~\cite{zhang2018efficient} to lower bound $-F(\xi_t)$ and $V(\xi_t) - \rho$ on $\mathcal{B}$; if one of the lower bound is nonnegative, \eqref{eq:verification} is verified. Otherwise, $\alpha,\!\beta$-CROWN conducts branch-and-bound: it splits $\mathcal{B}$ into smaller regions by cutting each dimension of $\mathcal{B}$ and solving verification subproblems in each subspace. The lower bounds tend to be tighter after branching, and \eqref{eq:verification} is verified when all subproblems are verified. We modified the branching heuristic on $\mathcal{B}$ to encourage branching at $\xi^*$, since $F(\xi_t)$ tends to be 0 around $\xi^*$, and tighter lower bounds are required to prove its positiveness. Compared to existing verifiers for neural Lyapunov certificates \cite{chang2019neural, dai2021lyapunov}, the efficient and GPU-friendly bound propagation procedure in $\alpha,\!\beta$-CROWN that exploits the structure of the verification problem is the key enabler to solving the difficult problem presented in~\eqref{eq:verification}. We can use bisection to find the largest 
sublevel set value $\rho_{\text{max}}$ that satisfies \eqref{eq:verification}. Our verification algorithm is outlined in \ref{algorithm:verification}. \looseness=-1

\begin{algorithm}
	\caption{Lyapunov-stable Neural Control Verification}
	\label{algorithm:verification}
	\begin{algorithmic}[1]
		\STATE \textbf{Input:} neural-network controller $\pi$, observer network $\phi_{\text{obs}}$, Lyapunov function $V$, sublevel set value estimate $\hat{\rho}_{\text{max}}$ (possibly from training), scaling factor $\lambda$, convergence tolerance $tol$
  \STATE \textbf{Output:} certified sublevel set value $\rho_{\text{max}}$
  \STATE \textit{// Find initial bounds $\rho_{\text{lo}}$ and $\rho_{\text{up}}$ for bisection}
  \STATE Verify \eqref{eq:verification} with ($\pi, \phi_{\text{obs}}, V, \hat{\rho}_{\text{max}}$) via $\alpha,\!\beta$-CROWN
  \IF{verified}
    \STATE $\rho_{\text{lo}} = \hat{\rho}_{\text{max}}$
    \STATE $\rho_{\text{up}} =$ multiply $\hat{\rho}_{\text{max}}$ by $\lambda$ until verification fails
  \ELSE
    \STATE $\rho_{\text{up}} = \hat{\rho}_{\text{max}}$
    \STATE $\rho_{\text{lo}} =$ divide $\hat{\rho}_{\text{max}}$ by $\lambda$ until verification succeeds
  \ENDIF
  \STATE \textit{// Bisection to find $\rho_{\text{max}}$}
  \WHILE{$\rho_{\text{up}}-\rho_{\text{lo}}>tol$}
  \STATE $\rho_{\text{max}} \gets \frac{\rho_{\text{lo}}+\rho_{\text{up}}}{2}$
  \STATE Verify \eqref{eq:verification} with ($\pi, \phi_{\text{obs}}, V, \rho_{\text{max}}$) via $\alpha,\!\beta$-CROWN \looseness=-1
  \IF{verified}
  \STATE $\rho_{\text{lo}} \gets \rho_{\text{max}}$
  \ELSE
  \STATE $\rho_{\text{up}} \gets \rho_{\text{max}}$
  \ENDIF
  \ENDWHILE
	\end{algorithmic}
\end{algorithm}
\setlength{\textfloatsep}{10pt}

\subsection{Training Formulation}
We adopt a new \emph{single-step} approach that can directly synthesize and verify the ROA. We define $H(\xi_{t+1})$ as the violation of $\xi_{t+1}$ staying within $\mathcal{B}$, which is positive for $\xi_{t+1}\notin \mathcal{B}$ and $0$ otherwise. Mathematically, for an axis-aligned bounding box $\mathcal{B} = \{\xi|\xi_{\text{lo}}\le \xi \le \xi_{\text{up}}\}$,
\begin{equation}
    H(\xi_{t+1}) = \|\text{ReLU}(\xi_{t+1} - \xi_{\text{up}})\|_1 + \|\text{ReLU}(\xi_{\text{lo}} - \xi_{t+1})\|_1.
\end{equation}


\begin{theorem}\label{theorem:training}
        The following conditions are necessary and sufficient for each other:
\begin{subequations}
\begin{equation} 
    (F(\xi_t) \le 0)  \land (H(\xi_{t+1}) \le 0)  \;\forall \xi_t\in\mathcal{S}\label{eq:lyapunov_derivative_roa_piecewise}  \Leftrightarrow
\end{equation}
\vspace{-0.5cm}
\begin{equation}
     \min (\text{ReLU}\left(F(\xi_t)\right) + c_0H(\xi_{t+1}), \ \rho-V(\xi_t)) \le 0 \;\forall \xi_t\in\mathcal{B}\label{eq:lyapunov_derivative_roa_piecewise}.
\end{equation}
\normalsize
\end{subequations}
\end{theorem}
Here, $c_0>0$ balances the violations of Lyapunov derivative condition and set invariance during training. The condition $H(\xi_{t+1})\le 0$ ensures $\mathcal{S}$ is invariant. Now we can synthesize the Lyapunov function and the controller satisfying condition \eqref{eq:lyapunov_derivative_roa_piecewise} over the explicitly defined domain $\mathcal{B}$.\footnote{To enforce \eqref{eq:lyapunov_derivative} in polynomial optimization, the S-procedure \cite{polik2007survey} from control theory employs finite-degree Lagrange multipliers to decide whether a  polynomial inequality $V(\xi_{t+1}) - V(\xi_t) \le -\kappa V(\xi_t)$ is satisfied over an invariant semi-algebraic set $\{\xi_t | V(\xi_t)< \rho\}$. In contrast, we can directly enforce \eqref{eq:lyapunov_derivative_roa_piecewise} thanks to the flexibility of $\alpha,\!\beta$-CROWN.} We define the violation on  \eqref{eq:lyapunov_derivative_roa_piecewise} as
\begin{equation}
\small
    L_{\dot{V}}(\xi_t; \rho) = \text{ReLU}(\min (\text{ReLU} (F(\xi_t)) + c_0 H(\xi_{t+1}), 
    \rho-V(\xi_t))).
    \label{eq:lyapunov_derivative_loss}
\end{equation}

The objective function \eqref{eq:roa_volume_objective} aims at maximizing the volume of $\mathcal{S}$. Unfortunately, the volume of this set cannot be computed in closed form. Hence, we seek a surrogate function that, when optimized, indirectly expands the volume of $\mathcal{S}$. Specifically, we select some candidate states $\xi^{(i)}_{\text{candidate}}, i=1,\hdots, n_{\text{candidate}}$ that we wish to stabilize with our controller. The controller and Lyapunov function are optimized to cover $\xi^{(i)}_{\text{candidate}}$ with $\mathcal{S}$, i.e., $V(\xi^{(i)}_{\text{candidate}}) < \rho$. Formally we choose to minimize this surrogate function
\begin{align}
L_{\text{roa}}(\rho) =\sum_{i=1}^{n_{\text{candidate}}} \text{ReLU}\left(\frac{V(\xi^{(i)}_{\text{candidate}})}{\rho}-1\right)\label{eq:roa_candidate_surrogate}
\end{align}
in place of maximizing the volume of $\mathcal{S}$ as in \eqref{eq:roa_volume_objective}. By carefully selecting the candidate states $\xi^{(i)}_{\text{candidate}}$, we can control the growth of the ROA. 
We discuss our strategy to select the candidates in more detail in Appendix \ref{subsec:candidate_selection}.
\subsection{Training Controller, Observer and Lyapunov Function} \label{subsec:training_alg}
We denote the parameters being searched during training as $\theta$, including:
\begin{itemize}[noitemsep,topsep=0pt,parsep=0pt,partopsep=0pt]
	\item The weights/biases in the controller network $\phi_\pi$;
	\item (NN Lyapunov function only) The weights/biases in the Lyapunov network $\phi_V$;
	\item The matrix $R$ in \eqref{eq:nn_lyapunov} or \eqref{eq:quadratic_lyapunov}.
    \item (Output feedback only) The weights/biases in the observer network $\phi_{\text{obs}}$. 
\end{itemize}

Mathematically we solve the problem \eqref{eq:lyapunov} through optimizing $\theta$ in the following program
\begin{subequations}
\setlength{\abovedisplayskip}{10pt} \setlength{\belowdisplayskip}{10pt}
\begin{align}
    \min_\theta &\text{ objective \eqref{eq:roa_candidate_surrogate}} \\
    \text{s.t } &\text{constraint \eqref{eq:lyapunov_derivative_roa_piecewise}},
\end{align}
\label{eq:lyapunov_bilevel}%
\end{subequations}
where \eqref{eq:lyapunov_positivity} and \eqref{eq:lyapunov_equilibrium} are satisfied by construction of the Lyapunov function.
Note that the constraint \eqref{eq:lyapunov_derivative_roa_piecewise} should hold for infinitely many $\xi_t\in\mathcal{B}$. To make this infinite-dimensional problem tractable, we adopt the Counter Example Guided Inductive Synthesis (CEGIS) framework \cite{abate2018counterexample,dai2020counter,ravanbakhsh2015counter}, which treats the problem \eqref{eq:lyapunov_bilevel} as a bi-level optimization problem. In essence, the CEGIS framework follows an iterative process. During each iteration, 
\begin{enumerate}[a., noitemsep,topsep=0pt,parsep=0pt,partopsep=0pt]
    \item \textit{Inner problem}: it finds counterexamples $\xi^i_{\text{adv}}$ by maximizing \eqref{eq:lyapunov_derivative_loss} .
    \item \textit{Outer problem}: it refines the parameters $\theta$ by minimizing a surrogate loss function across all accumulated (and hence, finitely many) counterexamples $\xi^i_{\text{adv}}$.
\end{enumerate}
This framework has been widely used in previous works to synthesize Lyapunov or safety certificates \cite{chang2019neural, dai2021lyapunov, abate2018counterexample, ravanbakhsh2015counter}. However, a distinct characteristic of our approach for complex systems is the avoidance of resource-intensive verifiers to find the worst case counterexamples. Instead, we use cheap projected gradient descent (PGD)~\cite{madry2017towards} to find counterexamples that violate~\eqref{eq:lyapunov_derivative_roa_piecewise}. We outline our training algorithm in \ref{algorithm:cegis}.
\begin{algorithm}
	\caption{Training Lyapunov-stable Neural Controllers}
	\label{algorithm:cegis}
	\begin{algorithmic}[1]
        \STATE \textbf{Input:} plant dynamics $f$ and $h$, region-of-interest $\mathcal{B}$, scaling factor $\gamma$, PGD stepsizes $\alpha$ and $\beta$, learning rate $\eta$
        \STATE \textbf{Output:} Lyapunov candidate $V$, controller $\pi$, observer $\phi_{\text{obs}}$ all in $\theta$
        \STATE Training dataset $\mathcal{D}=\varnothing$
            \FOR{$iter = 1, 2, \cdots$}
                \STATE Sample points $\bar{\xi}_j \in \partial \mathcal{B}$
                \FOR{$rho\textunderscore descent = 1, 2, \cdots$}
                \STATE $\bar{\xi}_j \leftarrow \text{Project}_{\partial\mathcal{B}}\left(\bar{\xi}_j - \alpha \cdot \frac{\partial V(\bar{\xi}_j)}{\partial\bar{\xi}}\right)$
                \ENDFOR
                \STATE $\rho = \gamma \cdot \min_j V(\bar{\xi}_j)$
                \STATE Sample counterexamples $\xi^i_{\text{adv}} \in \mathcal{B}$
                \FOR{$adv\textunderscore descent = 1, 2, \cdots$}
                \STATE $\xi^i_{\text{adv}} \leftarrow \text{Project}_{\mathcal{B}}\left(\xi^i_{\text{adv}} + \beta\cdot\frac{\partial 
    L_{\dot{V}}(\xi_{\text{adv}}^i; \rho)}{\partial\xi_{\text{adv}}}\right)$
                \ENDFOR
                \STATE $\mathcal{D} \leftarrow \{\xi_{\text{adv}}^i\} \cup \mathcal{D} $
                \FOR{$epoch = 1, 2, \cdots$}
                \STATE $\theta \leftarrow \theta - \eta \nabla_{\theta}L(\theta; \mathcal{D}, \rho)$
                \ENDFOR
            \ENDFOR
	\end{algorithmic}
\end{algorithm}
\setlength{\textfloatsep}{10pt}

\textbf{CEGIS within $\mathcal{S}$.} 
A major distinction compared to many CEGIS-based approaches is in line 12 and 16, where $L_{\dot{V}}$ only enforces the Lyapunov derivative constraint inside the certifiable ROA which depends on $\rho$. To encourage the sublevel set in \eqref{eq:S_def} to grow beyond $\mathcal{B}$, we parameterize $\rho = \gamma \cdot \min_{\bar{\xi}_j\in\partial \mathcal{B}}V(\bar{\xi}_j)$ with the scaling factor $\gamma > 1$. The largest $\gamma$ that leads to $\hat{\rho}_{\text{max}}$ can be found using bisection.
In line 5$-$9, we sample many points $\bar{\xi}_j$ on the periphery of $\mathcal{B}$, and apply PGD to minimize $V(\bar{\xi}_j)$. In line 10$-$14, we apply PGD again to maximize the violation \eqref{eq:lyapunov_derivative_loss} over randomly sampled $\xi^i_{\text{adv}}\in\mathcal{B}$ to generate a set of counterexamples in the training set $\mathcal{D}$. To make the training more tractable, we start with a small $\mathcal{B}$ and gradually grow its size to cover the entire region we are interested in.

\textbf{Loss functions for training.} In line 16 of Algorithm \ref{algorithm:cegis}, we define the overall surrogate loss function as
\begin{equation}
    L(\theta; \mathcal{D}, \rho) = \sum_{\xi^i_{\text{adv}}\in\mathcal{D}}L_{\dot{V}}(\xi^i_{\text{adv}}; \rho) + c_1 L_{\text{roa}}(\rho) + c_2 \|\theta\|_1 + c_3L_{\text{obs}},
    \label{eq:total_loss}
\end{equation}
where $c_1, c_2, c_3 > 0$ are all given positive constants. The term $L_{\dot{V}}(\bullet)$ is the violation on the Lyapunov derivative condition, defined in \eqref{eq:lyapunov_derivative_loss}; the term $L_{\text{roa}}$ is the surrogate loss for enlarging the region-of-attraction, defined in \eqref{eq:roa_candidate_surrogate}. To ease the difficulty of verification in the next step, we indirectly reduce the Lipschitz constant of the neural networks through the $l_1$ norm regularization $\|\theta\|_1$. Finally, we add $L_{\text{obs}}$ for output feedback case. We observe that it is important to explicitly regulate the observer performance during the training process. Otherwise, the training can easily diverge, and the observer will become unstable. In particular, we define the following observer performance loss
\begin{equation}
    L_{\text{obs}} = \sum_{\xi_t \in \mathcal{C}}\|\hat x_{t+1} - x_{t+1}\|_2,
\end{equation}
so that by minimizing this loss, the NN-based observer will try to predict the state at the next time step accurately. $\mathcal{C}$ is the set of randomly generated internal states within $\mathcal{B}$. 

\textbf{Discussions.} To obtain a certified ROA, prior CEGIS approaches on synthesizing neural-network Lyapunov functions \cite{chang2019neural, dai2021lyapunov} invoked expensive verifiers (SMT or MIP) during training to generate counterexamples. In contrast, we generate a large batch of counterexamples efficiently through the much cheaper PGD attack (requiring gradients only) on GPUs. Our results demonstrate that these cheap counterexamples are sufficient for guiding the training procedure, and the expensive verification process only needs to run once post-training. This finding is aligned with similar observations in the machine learning community \cite{balunovic2019adversarial,de2022ibp}. 
Our integration of heuristic PGD and sound verification combines the benefits of both.

\section{Experiments}
\label{section:results}
\vspace{-0.5em}
We demonstrate the effectiveness of our formulation 
in both verification and training. The proposed formulation leads to larger certifiable ROAs than previous works for multiple dynamical systems. 
The baseline approaches for comparison include: 1) discrete-time neural Lyapunov control (DITL) \cite{wu2023neural} which uses MIP for verification; 2) neural Lyapunov control (NLC) \cite{chang2019neural} employing SMT solvers for verification and counterexample generation; 3) neural Lyapunov control for unknown systems (UNL)\cite{zhou2022neural}.
Table \ref{table:our_runtime} reports the verification runtime for our trained models in Sec.~\ref{sec:exp_train_state_feedback} and~\ref{sec:exp_train_output_feedback}.

\begin{threeparttable}
\centering
\adjustbox{max width=.48\textwidth}{
\begin{tabular}{lcc}
\toprule
System &  Our runtime & DITL runtime\\
\midrule
Pendulum state &  11.3s\textsuperscript{\textdagger} & 7.8s \\
Path tracking & 11.7s & 13.3s \\
Cartpole & 129s & 448s\\
PVTOL\textsuperscript{\textdaggerdbl} & 217s & 1935s\\
\bottomrule
\end{tabular}
}
\begin{tablenotes}
\item[\textdagger] \footnotesize Runtime dominated by $\alpha,\!\beta$-CROWN startup cost.
\item[\textdaggerdbl] \footnotesize We discovered the verification implementation for PVTOL in~\cite{wu2023neural} missed certain regions in $\mathcal{B}$. But for a fair comparison of verification time, we used the same regions as theirs (see Sec.~\ref{sec:ditl_pvtol_bug}).
\end{tablenotes}
\vspace{-0.5em}
\caption{
Verification time comparison for models obtained using DITL. Our verification scales better than DITL to challenging environments because we do not use MIP solvers.
}
\label{table:comparison_runtime}
\vspace*{0.1cm}
\end{threeparttable}
\vspace{-0.5em}
\subsection{Verification of Existing Neural Lyapunov Models}
\vspace{-0.5em}
To show better scalability and larger ROAs of our verification algorithm, we first apply our verification procedure to models obtained using state-of-the-art DITL and compare against their verifier. Table \ref{table:comparison_runtime} records the verification runtime. All the system parameters and dynamics models are the same as in \cite{wu2023neural}.


Similar to~\cite{wu2023neural}, we visualize the ROA for two low-dimensional systems: 1) \textit{Inverted pendulum:} swing up the pendulum to the upright equilibrium $[\theta, \dot{\theta}]=[0, 0]$ with a state-feedback controller. 
2) \textit{Path tracking:} 
a planar vehicle tracks a given path with a state feedback controller. Our novel verification formulation discussed in Sec.~\ref{sec:verification_formulation} results in a larger ROA given the same models, as shown in Figure~\ref{fig:verification_roa_comparison}. Our ROA can nontrivially intersect the boundary of $\mathcal{B}$, represented as the borders of these figures, rather than only touching the boundary at a single point as in the previous approaches.
Compared to the MIP-based verification in DITL, Table~\ref{table:comparison_runtime} shows that our verification procedure offers significant advantages in verification runtime over DITL, especially on more challenging higher-dimensional tasks, such as Cart-pole and PVTOL.

\begin{figure}[t]
\includegraphics[width=0.48\textwidth]{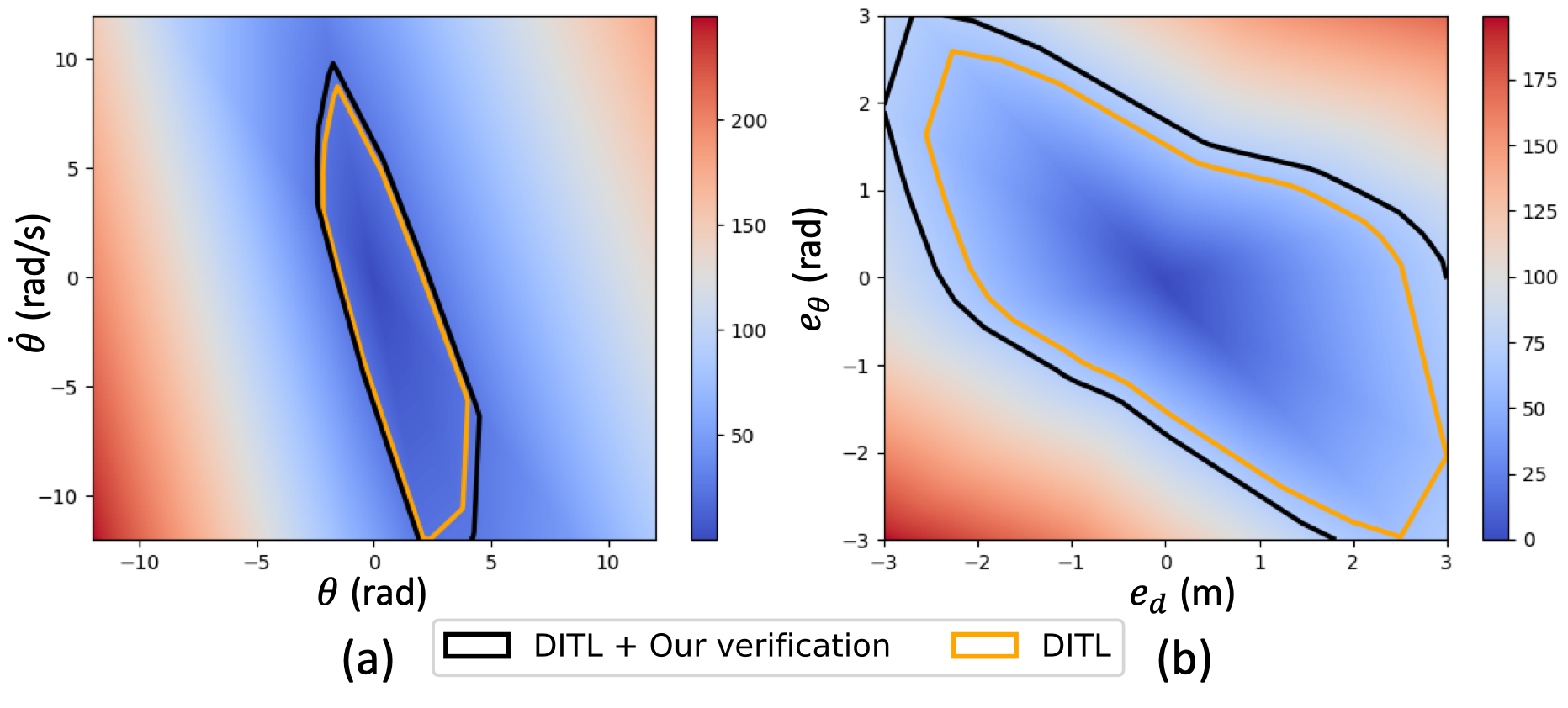}
\caption{Comparison of certified ROA for the \emph{same model} trained using DITL's code. Our novel verification formulation~\eqref{eq:verification} finds a larger ROA $\mathcal{S}$ (black) than $\Tilde{\mathcal{S}}$ (orange) in DITL for (a) inverted pendulum and (b) path tracking.}
\vspace*{-0.2cm}
\label{fig:verification_roa_comparison}
\end{figure}

\subsection{Training and Verification with New Formulation}
\label{sec:exp_train_state_feedback}
Our training and verification formulation, when combined, leads to even larger ROAs. We evaluate the effectiveness of our approach in the following state-feedback systems:
\begin{table}
\centering
\adjustbox{max width=.48\textwidth}{
\begin{tabular}{lc|cc}
\toprule
System &  Runtime & System &  Runtime\\
\midrule
Pendulum state &  33s  & Pendulum output & 94s\\
Quadrotor state &  1.1hrs & Quadrotor output &  8.9hrs \\
Path tracking & 39s  \\
\bottomrule
\end{tabular}}
\caption{
Verification runtime for our trained models. 
}
\label{table:our_runtime}
\end{table}
\begin{figure}
\includegraphics[width=0.48\textwidth]{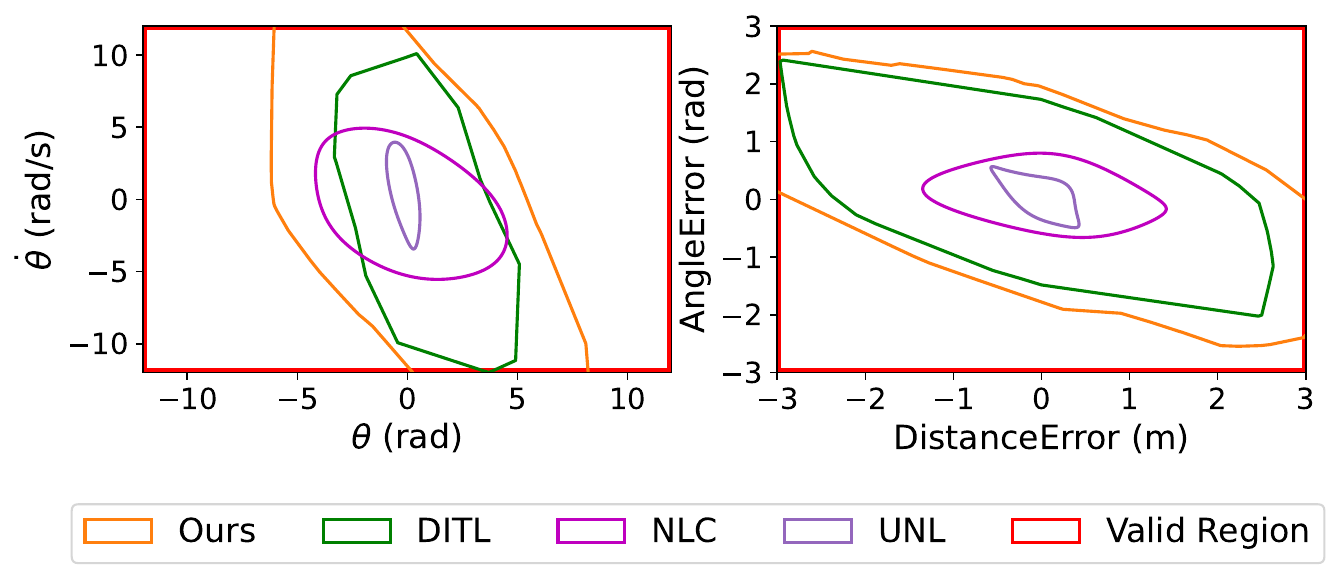}
\caption{Certified ROA for models trained with different methods. Our approach finds the largest ROA among all approaches for the inverted pendulum (left) and vehicle path tracking (right). The comparison to DITL uses the best Lyapunov network released by its original authors.
}
\label{fig:pendulum_path_tracking}
\end{figure}


\textbf{Inverted pendulum and path tracking.}
We compare our trained models for the inverted pendulum and path tracking against multiple baselines reported in \cite{wu2023neural}. Fig. \ref{fig:pendulum_path_tracking} shows that the ROA found by our improved formulation \eqref{eq:S_def} and \eqref{eq:rho_cond} is a strict superset of all the baseline ROAs. Again, our ROAs nontrivially intersect with the boundary of $\mathcal{B}$ (red borders), which is impossible with the formulation in prior works~\cite{chang2019neural, dai2021lyapunov, wu2023neural}. In Appendix \ref{subsec:pendulum_path_tracking_small_torque}, we present certified ROAs for both examples with more challenging torque limits.

\textbf{2D quadrotor.}
\begin{figure}
\centering
	\includegraphics[width=0.43\textwidth]{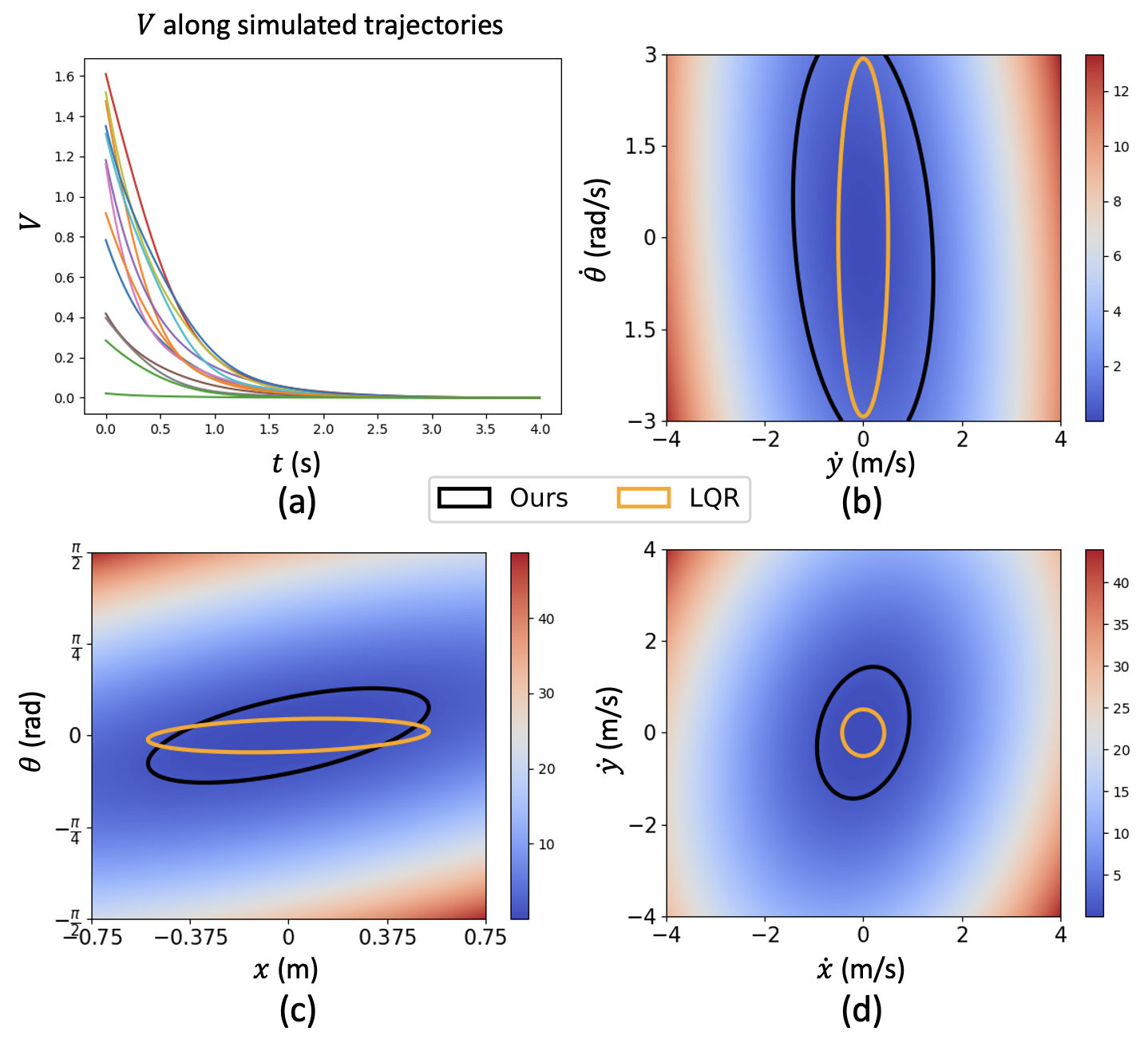}
	\caption{2D quadrotor state feedback: (a) the Lyapunov function keeps decreasing along simulated trajectories using our NN controller; (b-d) the verified region-of-attraction using our approach (black) and LQR (orange) in different 2-dimensional slices.}
	\label{fig:quadrotor_state}
\end{figure}
This is a more challenging setting not included in~\cite{wu2023neural}, where we aim to stabilize a quadrotor to hover at the equilibrium state $[x, y, \theta,\dot{x},\dot{y}, \dot{\theta}] = \mathbf{0}$. 
Our new formulation \eqref{eq:lyapunov_derivative_roa_piecewise} plays a crucial role in verifying this system. The previous formulation  \eqref{eq:lyapunov_derivative_box} enforces the Lyapunov derivative condition over the entire region $\mathcal{B}$, and we find that PGD attack can always find counterexamples during training.
With \eqref{eq:lyapunov_derivative_box}, the learned NN controllers are impossible to verify using the corresponding Lyapunov functions because violations can be detected even during training. In fact, \eqref{eq:lyapunov_derivative_box} requires $\mathcal{B}$ to lie within the true ROA that can be verified by $V$, which is not necessarily true for such a large $\mathcal{B}$ in the high-dimensional space. We simulate the system using our NN controller from various initial conditions within the verified ROA and observe that $V(x)$ always decreases along the simulated trajectories in Fig. \ref{fig:quadrotor_state}a. In Fig. \ref{fig:quadrotor_state}b$-$d, we visualize the certified ROA in different 2D slices and compare with that of the clamped LQR controller verified by the quadratic Lyapunov function obtained from the Riccati solution. 

\subsection{Neural Lyapunov Output Feedback Control}
\label{sec:exp_train_output_feedback}
We now apply our method to the more challenging output feedback control setting, which requires training a controller, an observer, and a Lyapunov function. For the first time in the literature, we demonstrate \emph{certified} neural Lyapunov control with \emph{output feedback} in two settings:

\textbf{Inverted pendulum with angle observation.}
For the output-feedback pendulum, the controller can only observe the angle $\theta$. 
Unlike \cite{chang2019neural, zhou2022neural, wu2023neural} which enforced an input constraint much larger than the gravity torque ($|u| \le 8.15 \cdot mgl$) for state-feedback pendulum, we impose the challenging torque limit $|u| \le \frac{mgl}{3}$. The black contours in Fig. \ref{fig:pendulum_output}a and \ref{fig:pendulum_output}b show a large verified ROA, whose corresponding sublevel set expands beyond $\mathcal{B}$. 

\begin{figure}
\centering
\includegraphics[width=0.43\textwidth]{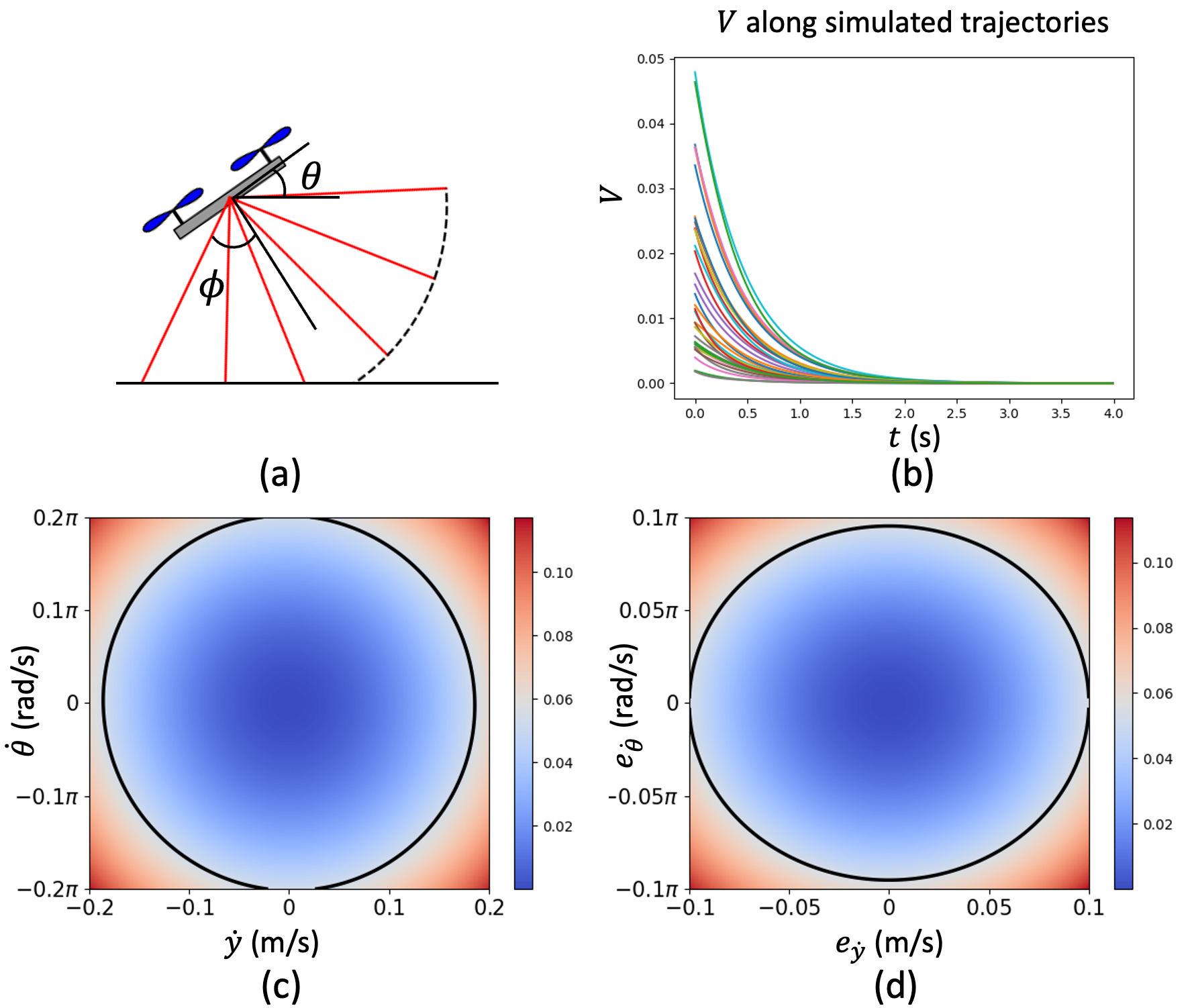}
	\caption{2D quadrotor output feedback: (a) the quadrotor obtains lidar measurements with 6 truncated rays; (b) the Lyapunov function keeps decreasing along simulated trajectories with our NN controller and observer; (c-d) the black contours represent 2-d slices of the certified ROA. Our approach provides the first formal neural certificate for the quadrotor output feedback system. }
	\label{fig:quadrotor_output_roa}
\end{figure}

\textbf{2D quadrotor with a lidar sensor.}
We validate our approach on a more complicated task of steering a 2D quadrotor to cruise at a constant height $y=0$ (the ground is at $y=-1\text{m}$) 
as visualized in Fig. \ref{fig:quadrotor_output_roa}a. 
The quadrotor obtains observations from a lidar sensor, which provides truncated distance along 6 rays in the range $\phi \in [-0.15\pi, 0.15\pi]$ up to a sensing horizon of 5m. 
We remark that SOS-based methods cannot handle such non-polynomial observation function with $\text{clamp}\left(\frac{y}{\cos(\theta - \phi)}, 0, 5 \right)$. 
Similar to the state feedback 2D quadrotor, we compare against the previous formulation \eqref{eq:lyapunov_derivative_box} and observe that verification is impossible since PGD attack can always find adversarial samples during training. In contrast, training using our formulation converges quickly to the stage where PGD attack cannot find adversarial samples. Fig. \ref{fig:quadrotor_output_roa}b demonstrates that the synthesized Lyapunov function using our approach keeps decreasing along the simulated trajectories using our Lyapunov-stable NN controller and observer. The black contours in Fig. \ref{fig:quadrotor_output_roa}c and \ref{fig:quadrotor_output_roa}d represent a decently large ROA verified by $\alpha,\!\beta$-CROWN. 

\section{Conclusion}
In this paper, we propose a novel formulation to efficiently synthesize and verify neural-network controllers and observers with Lyapunov functions, providing one of the earliest formal stability guarantees for output feedback systems in the literature. Our new formulation
actively promotes a large certifiable region-of-attraction. 
Distinct from prior works which rely on resource-intensive verifiers (e.g., SOS, MIP or SMT) to generate counterexamples during training, we incorporate cost-effective adversarial attacks that notably enhance training efficiency. 
Post-training, the Lyapunov conditions undergo a rigorous verification procedure
tailored for NN verification using $\alpha,\!\beta$-CROWN. 

\section*{Limitations}
While our method improves scalability for neural certificates by avoiding resource-intensive solvers for SOS, MIP, or SMT, the system dimensionality still poses a challenge for rigorous certification. Previous methods relying on expensive complete solvers were only able to handle \emph{state} feedback systems with lower dimensions: \cite{zhou2022neural} only dealt with 2-dimensional systems, \cite{chang2019neural} also suffered beyond 2 dimensions (errors and reproducibility issues are reported \href{https://github.com/YaChienChang/Neural-Lyapunov-Control/issues/14}{here}), and \cite{wu2023neural} scaled up to a 4-dimensional cartpole system (as noted in Appendix \ref{sec:ditl_pvtol_bug}, their corrected implementation failed for the 6-dimensional PVTOL). Although our approach extends neural certificates from state feedback to \emph{output} feedback control with 8 dimensions, the dimensions of the addressed systems remain moderate. We are interested in exploring the framework’s potential in higher dimensional systems with more complicated observation functions beyond the truncated lidar readings, such as images or point clouds.

\section*{Acknowledgement}
This work was supported by Amazon PO 2D-12585006, NSF 2048280, 2325121, 2244760, 2331966, 2331967 and ONR N00014-23-1-2300:P00001. Huan Zhang is supported in part by the AI2050 program at Schmidt Sciences (Grant \#G-23-65921) and NSF 2331967. The authors would like to thank Zico Kolter for valuable discussions and insightful feedback on the paper. 
\section*{Impact Statement}
This paper presents work whose goal is to advance the field of verification for neural network control with Lyapunov stability. Our work steps towards providing guarantees for real-world safety-critical control applications.

\bibliography{reference}
\bibliographystyle{icml2024}

\newpage
\appendix
\onecolumn
\section{Proofs}
\subsection{Proof of Theorem \ref{theorem:verification}}
\begin{proof}
\begin{subequations}
\begin{align}
    &(-F(\xi_t) \ge 0  \land  \xi_{t+1}\in \mathcal{B}) \ \lor \ (V(\xi_t) \ge \rho), \ \forall \xi_t\in\mathcal{B}\label{eq:verification_proof1}\\
    \Leftrightarrow& (-F(\xi_t)\ge 0 \land \xi_{t+1}\in \mathcal{B}), \forall (\xi_t\in\mathcal{B}\land V(\xi_t) < \rho)\label{eq:verification_proof2}\\
    \Leftrightarrow &(V(\xi_{t+1}) - V(\xi_t)\le -\kappa V(\xi_t)\land \xi_{t+1}\in\mathcal{B}),\forall \xi_t\in\mathcal{S} \\
    \Leftrightarrow &(V(\xi_{t+1})-V(\xi_t)\le -\kappa V(\xi_t)\land \xi_{t+1}\in\mathcal{B}\land V(\xi_{t+1}) < \rho),\forall \xi_t\in\mathcal{S} \label{eq:22d}\\
    \Leftrightarrow& (V(\xi_{t+1})-V(\xi_t)\le -\kappa V(\xi_t) \land \xi_{t+1}\in\mathcal{S}),\forall \xi_t\in\mathcal{S}
\end{align}
\end{subequations}
Hence $\mathcal{S}$ is an invariant set and the function $V$ decreases exponentially within this invariant set, which proves stability, and $\mathcal{S}$ as an inner approximation of the ROA. The appearance of $V(\xi_{t+1}) < \rho$ in \eqref{eq:22d} arises from the fact that $V(\xi_{t}) < \rho,\forall \xi_t\in\mathcal{S}$ and $V(\xi_{t+1}) \le V(\xi_{t}) < \rho$ by \eqref{eq:lyapunov_derivative}.
\end{proof}
\subsection{Proof of Theorem \ref{theorem:training}}
\begin{proof}
\begin{subequations}
\begin{align}
    &\min (\text{ReLU}\left(F(\xi_t)\right) + c_0H(\xi_{t+1}), \ \rho-V(\xi_t)) \le 0 \;\forall \xi_t\in\mathcal{B} \\
    \Leftrightarrow &(\text{ReLU}\left(F(\xi_t)\right) + c_0H(\xi_{t+1}) \le 0) \lor (\rho-V(\xi_t) \le 0) \;\forall \xi_t\in\mathcal{B} \\ 
    \Leftrightarrow &(\text{ReLU}\left(F(\xi_t)\right) \le 0 \land c_0H(\xi_{t+1}) \le 0) \lor (\rho-V(\xi_t) \le 0) \;\forall \xi_t\in\mathcal{B} \label{eq:nonnegative_relu_h}\\
    \Leftrightarrow &(F(\xi_t) \le 0 \land H(\xi_{t+1}) \le 0) \lor (\rho-V(\xi_t) \le 0) \;\forall \xi_t\in\mathcal{B} \\
    \Leftrightarrow &(F(\xi_t) \le 0 \land H(\xi_{t+1}) \le 0) , \forall (\xi_t\in\mathcal{B}\land V(\xi_t) < \rho) \\
    \Leftrightarrow &(F(\xi_t) \le 0 \land H(\xi_{t+1}) \le 0) , \forall \xi_t\in\mathcal{S}
\end{align}
\end{subequations}
\eqref{eq:nonnegative_relu_h} follows from the fact that both $\text{ReLU}\left(F(\xi_t)\right)$ and $H(\xi_{t+1})$ are nonnegative.
\end{proof}
\section{Experiment Details}
\begin{table}
\centering
\begin{tabular}{lccccc}
\toprule
System & Feedback & Lyapunov function & controller & observer & Region-of-interest (upper limit)\\
\midrule
Pendulum & State & (16, 16, 8) & (8, 8, 8, 8) & --- & $[12, 12]$\\ 
Path tracking & State & (16, 16, 8) & (8, 8, 8, 8)& --- & $[3, 3]$\\ 
Quadrotor & State & Quadratic & (8, 8) & --- & $[0.75, 0.75, \frac{\pi}{2}, 4, 4, 3]$\\ 
Pendulum & Output  & Quadratic & (8, 8, 8) & (8, 8)& $[0.4\pi, 0.4\pi, 0.1\pi, 0.1\pi]$\\ 
Quadrotor & Output & Quadratic & (8, 8) &  (8, 8) & $[0.1, 0.2\pi,0.2, 0.2\pi, 0.05, 0.1\pi, 0.1, 0.1\pi]$\\
\bottomrule
\end{tabular}
\caption{Neural network size and region-of-interest for each task. The tuples denote the number of neurons in each layer of the neural network. All the networks use the leaky ReLU activation function.
}
\label{table:nn_size}
\end{table}
\subsection{Candidate State Selection for Growing ROA}\label{subsec:candidate_selection}
On the one hand, the candidate states that we hope to be covered in the invariant set $\mathcal{S}$ should be diverse enough to encourage the ROA to grow in all directions; on the other hand, they should not be irregularly spread across the entire state space because such candidates might shape the ROA in conflicting directions and deteriorate the satisfaction of the Lyapunov derivative condition \eqref{eq:lyapunov_derivative}. We require the candidate states to have the same distance from the goal state in the metric of the Lyapunov function value and start by sampling states on the 1-level set of a reference Lyapunov function $V_\text{ref}$. For state feedback, we choose $V_\text{ref}$ to be the LQR Lyapunov function $x^T S x$ ($S$ is the solution to the Riccati equation); for output feedback, we select $V_\text{ref} = x^T S x + e^T P^{-1}e$ ($P$ is the asymptotic state variance at the goal state obtained by solving the discrete Riccati equation). After the NN Lyapunov function is trained to achieve a reasonable ROA,  we can sample states slightly outside the current ROA as candidates.

\subsection{Pendulum State Feedback \& Path Tracking with Challenging Torque Limits}\label{subsec:pendulum_path_tracking_small_torque}
\begin{figure}
\centering
	\includegraphics[width=0.48\textwidth]{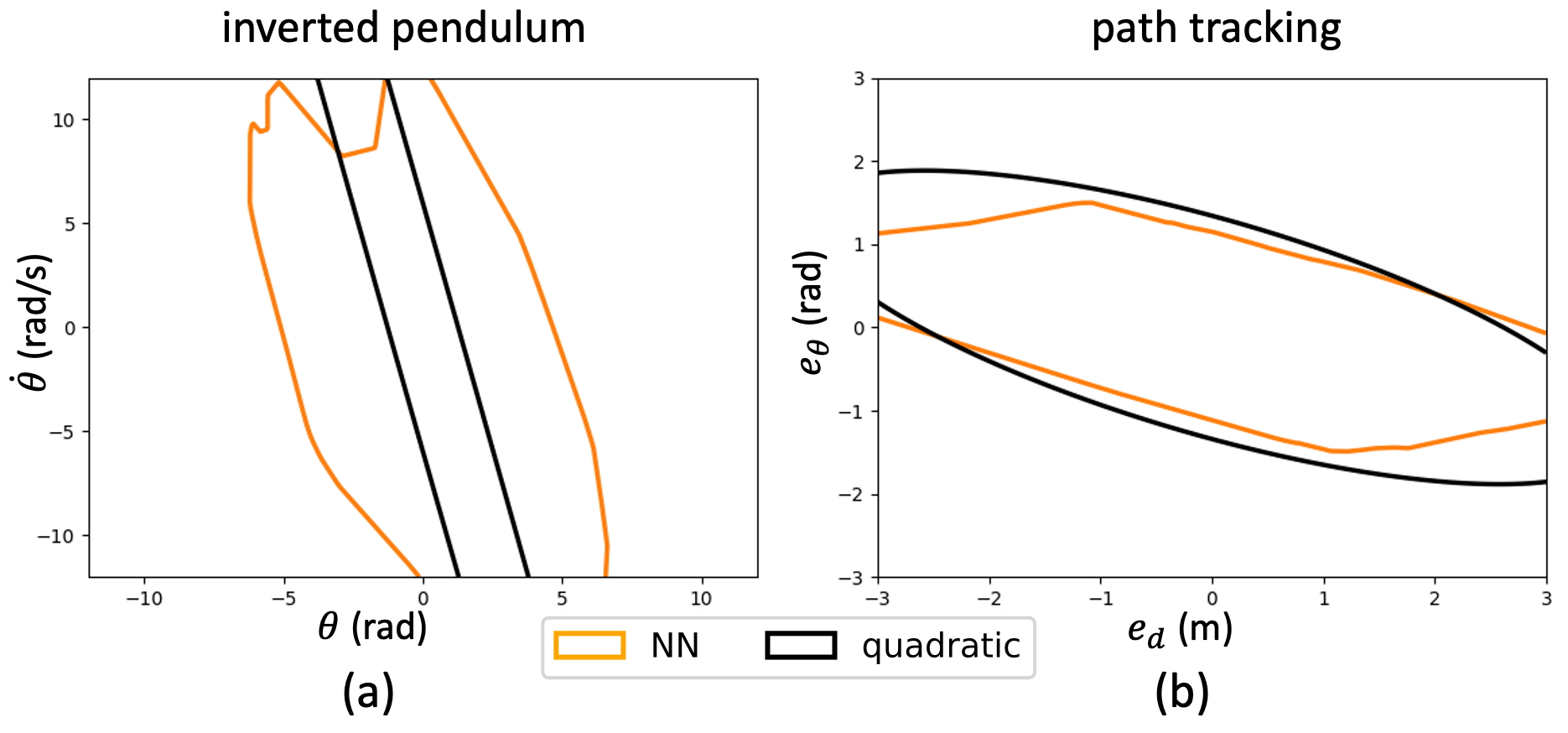}
	\caption{Comparison of certified ROAs obtained using NN and quadratic Lyapunov functions: (a) inverted pendulum with torque limit $|u|\le 1.02 \cdot mgl$ (b) path tracking with $|u|\le \frac{L}{v}$.}
	\label{fig:pendulum_path_tracking_small_torque}
\end{figure}
In Sec. \ref{sec:exp_train_state_feedback}, we provide certified ROAs for inverted pendulum and path tracking with large (easy) torque limits as a fair comparison to all the baselines ($|u| \le 8.15 \cdot mgl$ for pendulum and $|u| \le 1.68\frac{L}{v}$ for path tracking). In Fig. \ref{fig:pendulum_path_tracking_small_torque}, we demonstrate ROAs for small (challenging) torque limits verified with both neural and quadratic Lyapunov functions. While neural Lyapunov functions can be more expressive, quadratic Lyapunov functions are often easier to train and have better interpretability. Our approach aims to leverage the strengths of both representations, allowing practitioners to select the most suitable form based on their specific requirements and trade-offs between expressivity, convergence, and interpretability.
\subsection{Pendulum Output Feedback}
In Fig. \ref{fig:pendulum_output_umax}, we visualize the phase portrait and certified ROA with a larger torque limit $|u|\le 1.36 \cdot mgl$ . We synthesize a quadratic Lyapunov function in the region-of-interest $\pm[0.7\pi, 0.7\pi, 0.175\pi, 0.175\pi]$ and an NN Lyapunov function in $\pm[\pi, \pi, 0.25\pi, 0.25\pi]$. With such a large control constraint, the phase portrait demonstrates that starting from many initial states (even outside the verified ROA), the system can always converge to the upright equilibrium with the synthesized controller and observer. This result suggests that our novel loss function \eqref{eq:total_loss} both leads to a large certified ROA and enables good generalization.
\begin{figure}
\centering
	\includegraphics[width=0.96\textwidth]{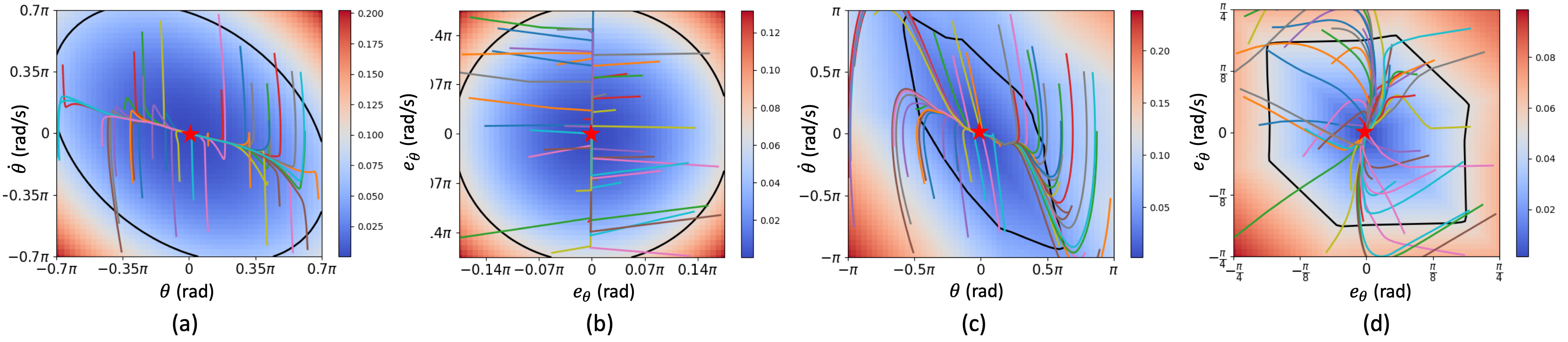}
	\caption{Pendulum output feedback: phase portrait and certified ROA with torque limit $|u|\le 1.36 \cdot mgl$ using (a)-(b) quadratic Lyapunov function and (c)-(d) neural-network Lyapunov function.}
	\label{fig:pendulum_output_umax}
\end{figure}
\subsection{2D Quadrotor Output Feedback}
In Fig. \ref{fig:quadrotor_output_snapshot}, we visualize the snapshots of the quadrotor stabilized by our NN controller and observer with decently large initial state estimation error. We observe that the NN controller and observer generalize well outside of the certified ROA, empirically steering the quadrotor to cruise at the constant height for most of the states within the box region.
\begin{figure}
\centering
	\includegraphics[width=0.36\textwidth]{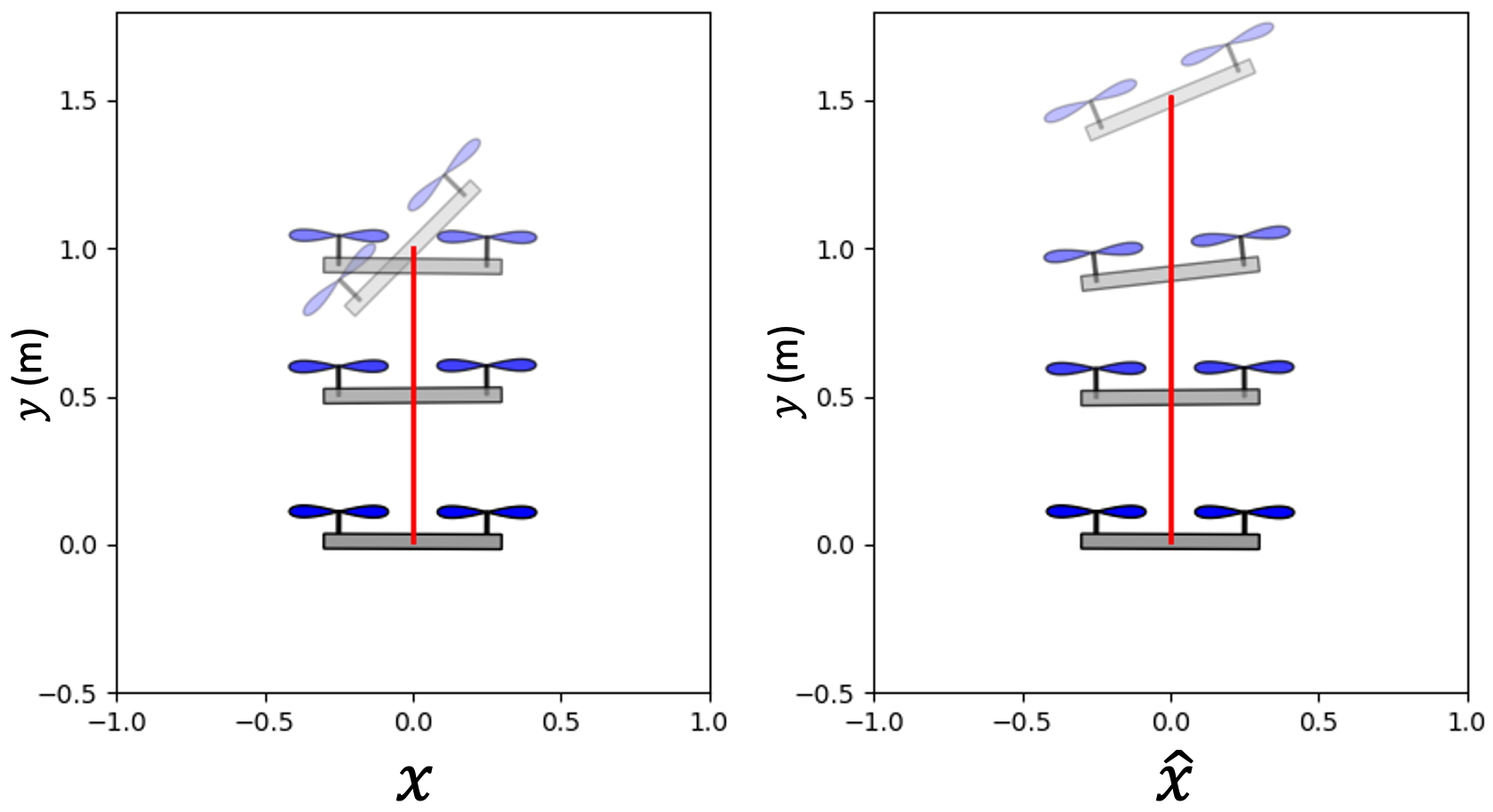}
	\caption{Snapshots of simulating 2D quadrotor with lidar sensor using our NN controller and observer. The red curve is the quadrotor's trajectory to balance at $y=0$. The left plot illustrates the true states and the right plot shows the state estimates. The initial state estimate error is $[0.5, -\frac{\pi}{8}, 0.1, \frac{\pi}{8}]$.}
	\label{fig:quadrotor_output_snapshot}
\end{figure}
\subsection{Validation region $\mathcal{B}$ in Fig.\ref{fig:pendulum_path_tracking}}
We use the validation region $\mathcal{B}$ as reported in each paper, detailed in Table \ref{table:baseline_b_size}.
\begin{table}[ht]
\centering
\begin{tabular}{lcc}
\toprule
& Inverted Pendulum & Path tracking\\ 
\midrule
Ours & $\|x\|_\infty \le 12$ & $\|x\|_\infty \le 3$\\ 
DITL & $\|x\|_\infty \le 12$ & $\|x\|_\infty \le 3$\\
NLC & $\|x\|_2 \le 6$ & $\|x\|_2 \le 1.5$\\
UNL & $\|x\|_2 \le 4$ & $\|x\|_2\le 0.8$\\
\bottomrule
\end{tabular}
\caption{Validation region $\mathcal{B}$ in each approach.
}
\label{table:baseline_b_size}
\end{table}
\subsection{Region for PVTOL in \cite{wu2023neural}}
\label{sec:ditl_pvtol_bug}
We mentioned in \Cref{table:comparison_runtime} that there is an implementation issue in \cite{wu2023neural} regarding region $\mathcal{B}$ for PVTOL. \cite{wu2023neural} takes $0.1\leq\|x\|_\infty\leq 1$ for $\mathcal{B}$, where $\|x\|_\infty$ should be the \emph{maximum} absolute value among all the dimensions in $x$. We found that the code of \cite{wu2023neural}\footnote{\url{https://github.com/jlwu002/nlc_discrete/blob/main/pvtol.py}} mistakenly implemented $\|x\|_\infty$ as the \emph{minimum} absolute value among all the dimensions in $x$ when enforcing the  $\|x\|_\infty\geq 0.1$ constraint, which makes the resulting $\mathcal{B}$ much smaller than desired. 
We found the issue in their code released by 12/10/2023, and we were able to reproduce the results on their paper using this version of code with an incorrect $\mathcal{B}$. While the implementation issue has been fixed in their current version of code released on 12/29/2023, we found that the new version is not able to successfully finish training the model on PVTOL with the correct $\mathcal{B}$. 

\end{document}